\documentclass{article}

\usepackage{microtype}
\usepackage{graphicx}
\usepackage{subfigure}
\usepackage{booktabs} 

\usepackage{hyperref}

\usepackage[accepted]{icml2021}

\icmltitlerunning{Recurrence of Optimum for Training Weight and Activation Quantized Networks}
\usepackage[switch]{lineno}
\usepackage{graphicx}
\usepackage{amsmath,bm}
\usepackage{amssymb}
\usepackage{amsthm}
\usepackage{appendix}
\usepackage{dsfont}
\usepackage{fancyvrb}
\usepackage[utf8]{inputenc}
\usepackage{optidef}
\usepackage{parskip}
\usepackage{tikz}
\usepackage{pgfplots}
\pgfplotsset{compat=1.15}
\usetikzlibrary{calc,intersections}
\usetikzlibrary{positioning}

\newtheorem*{rep@theorem}{\rep@title}
\newcommand{\newreptheorem}[2]{%
\newenvironment{rep#1}[1]{%
 \def\rep@title{#2 \ref{##1}}%
 \begin{rep@theorem}}%
 {\end{rep@theorem}}}
\makeatother
\newreptheorem{prop}{Proposition}
\newreptheorem{lemma}{Lemma}
\newreptheorem{thm}{Theorem}
\newreptheorem{cor}{Corollary}
\newreptheorem{ex}{Example}

\newtheorem{definition}{Definition}
\newtheorem{theorem}{Theorem}
\newtheorem{lemma}{Lemma}
\newtheorem{prop}{Proposition}
\newtheorem{rmk}{Remark}

\newtheorem{cor}{Corollary}
\newtheorem{ex}{Example}
\DeclareMathOperator*{\argmax}{arg\,max}

\newcommand{\A}[0]{~~\text{and}~~}
\newcommand{\norm}[1]{\left\|#1\right\|}
\newcommand{\normv}[1]{\frac{#1}{\left\|#1\right\|}}
\newcommand{\inner}[2]{\left\langle#1,#2\right\rangle}
\newcommand{\bsign}[1]{\widetilde{\mathrm{sign}}\left(#1\right)}
\newcommand{\sign}[1]{\mathrm{sign}\left(#1\right)}
\newcommand{\Q}[2]{\mathcal{Q}_{#1}^{#2}}
\newcommand{\tQ}[2]{\tilde{\mathcal{Q}}_{#1}^{#2}}
\newcommand{\proj}[2]{\mathrm{proj}_{\mathcal{Q}_{#1}^{#2}}}
\newcommand{\nproj}[2]{\widetilde{\mathrm{proj}}_{\mathcal{Q}_{#1}^{#2}}}

\newcommand{\set}[1]{\left\{#1\right\}}
\newcommand{\ind}[1]{\mathds{1}_{\left\{#1\right\}}}
\newcommand{\constv}[0]{\frac{\norm{\bm v}^2}{2\sqrt{2\pi}}}
\newcommand{\Z}{\bm Z}
\newcommand{\y}{\bm y}
\newcommand{\w}{\bm w}
\newcommand{\x}{\bm x}
\newcommand{\E}{\mathbb{E}}

\newcommand{\cQ}{\mathcal{Q}}
\newcommand{\R}{\mathbb{R}}

\begin{document}
\twocolumn[
\icmltitle{Recurrence of Optimum for Training Weight and Activation Quantized Networks}



\icmlsetsymbol{equal}{*}

\begin{icmlauthorlist}
\icmlauthor{Ziang Long}{equal,uci}
\icmlauthor{Penghang Yin}{equal,abn}
\icmlauthor{Jack Xin}{uci}
\end{icmlauthorlist}

\icmlaffiliation{uci}{University of California, Irvine}
\icmlaffiliation{abn}{University at Albany, State University of New York}

\icmlcorrespondingauthor{Ziang Long}{zlong6@uci.edu}

\icmlkeywords{Machine Learning, ICML}

\vskip 0.3in
]




\begin{abstract}
Deep neural networks (DNNs) are quantized  for efficient inference on resource-constrained platforms. However, training deep learning models with low-precision weights and activations involves a demanding optimization task, which calls for minimizing a stage-wise loss function subject to a discrete set-constraint.
While numerous training methods have been proposed, existing studies for full quantization of DNNs are mostly empirical. From a theoretical point of view, we study practical techniques for overcoming the combinatorial nature of network quantization. Specifically, we investigate a simple yet powerful projected gradient-like algorithm for quantizing two-layer convolution networks, by repeatedly moving one step at float weights in the negative direction of a heuristic \emph{fake} gradient of the loss function (so-called coarse gradient) evaluated at quantized weights.
For the first time, we prove that under mild conditions, the sequence of quantized weights recurrently visit the global optimum of the discrete minimization problem for training fully quantized network. We also show numerical evidence of the recurrence phenomenon of weight evolution in training quantized deep networks. 
\end{abstract}

\section{Introduction}
Deep neural networks (DNNs) have been profoundly transforming machine learning, in applications of computer vision, reinforcement learning, and natural language processing, and so on. While achieving human level or even super-human performances, DNNs typically have tremendous number of weights with high resource consumption at inference time, which poses a challenge for their deployment on mobile devices used in our daily lives. To address this challenge, research efforts have been made to the quantizing weights and activations of DNNs while maintaining their performance. Quantization methods train DNNs with the weights and activation values being constrained to low-precision arithmetic rather than the conventional floating-point representation in full-precision. \cite{Hubara2017QuantizedNN,dorefa_16,halfwave_17,inq_17,louizos2018relaxed,ttq_16}, which offer the feasibility of running DNNs on CPUs rather than GPUs in real-time. For example, the XNOR-Net \cite{rastegari2016xnor} with binary weights and activations sees 58$\times$ faster convolutional operations and 32$\times$ memory savings.

Training fully quantized DNN requires solving a challenging optimization problem with piecewise constant (and non-convex) training loss functions and a discrete set-constraint.
That is, one considers the following constrained optimization problem for training quantized neural nets:
\begin{equation}\label{training}
    \min_{\w} \; f(\w):=\E_{\x\sim p(\x)}[\ell(\w; \x)] \quad \mbox{subject to} \quad \w \in \cQ
\end{equation}

where $\ell(\w;\x)$ is the loss function for sample $\x$, which is \emph{discrete-valued} as non-linear activations are also quantized;  $\cQ$ is the set of quantized weights. For general constrained  minimization, the classical projected gradient descent (PGD): 
\begin{equation*}
\w^{t+1} = \proj{}{}\left(\w^t - \eta_t \, \E[\nabla_{\w} \ell(\w^t;\x)]\right) 
\end{equation*}
is considered.
Here $\proj{}{}$ is the projection onto set $\cQ$ for quantizing float weights to ones at low bit-width, giving a weight quantization scheme. However, with quantized activations, the gradient of loss function $\nabla_{\w} \ell(\w;\x)$ is almost everywhere (a.e.) zero, leaving the standard back-propagation and hence PGD inapplicable. 

In this paper, we study the following iterative algorithm for training fully quantized networks
\begin{equation}\label{coarse}
\left\{
\begin{aligned}
&\y^{t+1} = \y^t - \eta_t \, \E[\tilde{\nabla}_{\w} \ell(\w^t;\x)]\\
&\w^{t+1} = \proj{}{}(\y^{t+1}), 
\end{aligned}
\right.
\tag{QUANT}
\end{equation}
where $\tilde{\nabla}_{\w} \ell$ denotes some heuristic modification of the vanished $\nabla_{\w} \ell$ based on the so-called straight-through estimator (STE) \cite{bengio2013estimating,hinton2012neural}, rendering a valid search direction. Following \cite{yin2018understanding} , we shall refer to this fake `gradient' induced by STE as coarse gradient throughout this paper. Compared with PGD which can be recast as the two-step iteration: 
\begin{equation}
\left\{
\begin{aligned}
&\y^{t+1} = \w^t - \eta_t \, \E[\nabla_{\w}\ell(\w^t;\x)]\\ 
&\w^{t+1} = \proj{}{}(\y^{t+1})
\end{aligned}
\right.
\tag{PGD}
\end{equation}
another key difference is that, in the gradient step, float weights $\y^{t+1}$ is updated by perturbing $\y^t$ instead of the current projection $\w^t$.

\subsection{Related works}
For the best possible performance under quantization, the pre-trained full-precision  networks need to be re-trained.
In the regime of weight quantization, the BinaryConnect scheme:
\begin{equation}\label{bc}
\left\{
\begin{aligned}
&\y^{t+1} = \y^t - \eta_t \, \E[\nabla_{\w} \ell(\w^t;\x)] \\ &\w^{t+1} = \proj{}{}(\y^{t+1})
\end{aligned}
\right.
\end{equation}
was first proposed in \cite{courbariaux2015binaryconnect} for training DNNs with binary (1-bit) weights. It is similar to \ref{coarse}, but simply uses the standard gradient $\nabla_{\w} \ell$ as the activation values were not quantized. The method was then extended to multi-bit weight quantization such as ternary weight networks \cite{twn_16}. 
On the theoretical side, \cite{li2017training} analyzed the convergence of BinaryConnect scheme for weight quantization, and proved that $\{\w^{t}\}$ converge to an error floor region of the optimal quantized weights under strong convexity and smoothness assumptions on $f$. Recently, \cite{Lin2020Dynamic} used an algorithm called ``error feedback" for pruning networks \cite{han2015deep,xiao2019autoprune}. It is basically the same as BinaryConnect, except that the weight quantization step $\proj{}{}$ is replaced with weight pruning/thresholding which can also be viewed as a projection. The authors showed the convergence to a neighborhood of optimal solution under strong convexity and smoothness assumptions whose radius is $O(\sqrt{d})$ with $d$ being the number of model parameters. Moreover, it remains unclear whether the global optimum can actually be reached in this setting. 

The idea of STE  has been extensively used for efficiently handling discrete-valued functions arising in machine learning problems. A STE, used in the backward pass only, is a heuristic proxy that substitutes the a.e. zero derivative of discrete component composited in the loss function when computing the gradient under chain rule. Its applications include, but are not limited to, network quantization \cite{bnn_16,halfwave_17,dorefa_16,pact,Hubara2017QuantizedNN,uhlich2020mixed,blumenfeld2019mean}, neural architecture search\cite{stamoulis2020single},  knowledge graphs \cite{xu2019relation},  discrete latent representations \cite{jang2016categorical}. For networks with binary activations (and real-valued weights), \cite{yin2018understanding} showed that STE-based gradient (called coarse gradient) methods converge only when a proper STE like ReLU STE \cite{halfwave_17} is used. And they proved that the negation of the resulting coarse gradient points to a descent direction that makes the training loss decrease. For quantization of both weights and activations, \cite{bnn_16,Hubara2017QuantizedNN,halfwave_17,pact,dorefa_16} utilized \ref{coarse} scheme which is the combination of BinaryConnect and STE, and achieved state-of-the-art classification accuracies. Yet to our knowledge, no convergence results of \ref{coarse} have been established to date. 

\subsection{Main contributions}
In this paper, we examine the quantization of one-hidden-layer networks with binary activation and binary or ternary weights using the \ref{coarse} algorithm. Surprisingly, the sequence of quantized weights $\{\w^t\}$ generated by \ref{coarse} is generically divergent.  Our key contributions are the \emph{first} groundbreaking theoretical results on the dynamics of \ref{coarse} algorithm for learning fully quantized neural nets: 
(1) we prove the generic divergence if the teacher parameters are not in a quantized state, and give an explicit example of oscillatory divergence behavior (the sequence $\{\w^t\}$ has period 3 and jumps between sub-optimal quantized states; see Example 1). 
(2) We explicitly point out, in the ternary case, the $n$ (out of $3^n-1$) sub-optimal quantized states that $\{\w^t\}$ could visit infinite many times; see Remark \ref{suboptimals} and Lemma \ref{choices}.
(3) We prove that $\{\w^t\}$ oscillates around the global optimum of quantization problem. Under conditions that teacher parameters and their quantized values are close enough (see Theorem \ref{recurrent}), $\{\w^t\}$ visits the quantized teacher parameters (the optimum) infinitely often ({\it recurrence}). 
Compared with theoretical results for BinaryConnect \cite{li2017training,Lin2020Dynamic}, our analysis is more precise and in depth in order to overcome  a biased gradient modification in \ref{coarse} based on straight-through estimator (STE) \cite{hinton2012neural,bengio2013estimating}. Our result is stronger in that the recurrence behavior at global minimum holds {\it without global convexity assumption of the loss function}.

\textbf{Organization.} In section 2, we introduce the problem setup and present some useful preliminary results about the \ref{coarse} algorithm. In section 3,  we summarize the main results regarding the recurrence behavior of \ref{coarse} algorithm. More technical details and sketch of proofs are presented in section 4. 

\section{Preliminaries}
\subsection{Problem Setup}

We consider a one-hidden-layer model that outputs the prediction for an input $\bm Z\in\mathbb{R}^{m\times n}$:
\begin{equation}\label{network}
y(\bm Z;\bm w):=\sum_{i=1}^mv_i\sigma\left(\bm Z_i^\top\bm w\right)=\bm v^\top\sigma\left(\bm Z\bm w\right)
\end{equation}
where $\bm Z_i^\top$ denotes the $i$-th row vector of $\bm Z$;
$\bm w\in\mathbb{R}^n$ is the trainable weights in the first linear layer, and $\bm v\in\mathbb{R}^m$ the weights in the second linear layer which are assumed to be known and fixed during the training process;  the activation function $\sigma(x)=\mathds{1}_{\{x>0\}}$ is \emph{binary}, acting component-wise on the vector $\bm Z\bm w$. The label is generated according to $y^*_{\Z} := y(\bm Z;\bm w^*)$ for some unknown teacher (real-valued) parameters $\bm w^*\in\mathbb{R}^n$. 

\tikzset{%
   neuron missing/.style={
    draw=none, 
    scale=2,
    text height=0.333cm,
    execute at begin node=\color{black!20}$\vdots$
  },
}

\tikzset{%
   neuron missing2/.style={
    draw=none, 
    scale=2,
    text height=0.333cm,
    execute at begin node=\color{black!20}$\cdots$
  },
}

\tikzset{%
   neuron missing3/.style={
    draw=none, 
    scale=2,
    text height=0.333cm,
    execute at begin node=\color{black!20}$\ddots$
  },
}

\tikzset{%
   neuron missing4/.style={
    draw=none, 
    scale=2,
    text height=0.333cm,
    execute at begin node=\color{black!50}$\vdots$
  },
}

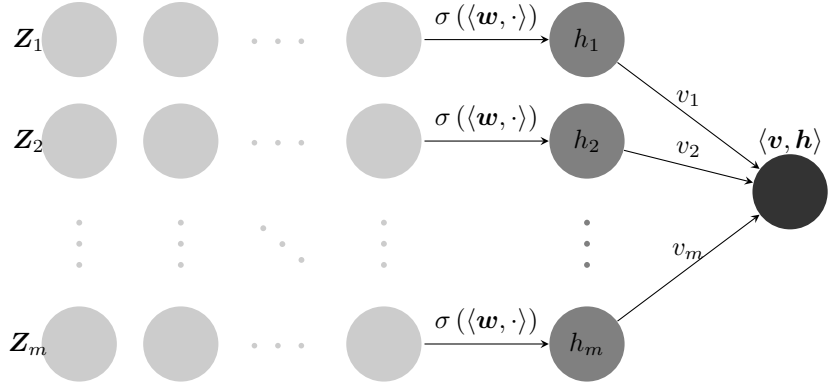
\begin{figure*}[ht]
\centering
\begin{tikzpicture}[scale=0.9, x=1.5cm, y=1.5cm, >=stealth]

\foreach \m/\l [count=\y] in {1,2}
{
 \node [circle,fill=black!20,minimum size=1cm] (input-\m) at (-3,2-\y) {};
 \node [circle,fill=black!20,minimum size=1cm] (input-\m) at (-2,2-\y) {};
 \node [circle,fill=black!20,minimum size=1cm] (input-\m) at (0,2-\y) {};
 \node [circle,fill=white!50,minimum size=0.1cm] (input-\m) at (-3.5,2-\y) {};
}
\foreach \m/\l [count=\y] in {3}
{
 \node [circle,fill=black!20,minimum size=1cm] (input-\m) at (-3,-2) {};
 \node [circle,fill=black!20,minimum size=1cm] (input-\m) at (-2,-2) {};
 \node [circle,fill=black!20,minimum size=1cm] (input-\m) at (0,-2) {};
 \node [circle,fill=white!50,minimum size=0.1cm] (input-\m) at (-3.5,-2) {};
}
 
 \node [neuron missing2] at (-1,1.1){};
 \node [neuron missing2] at (-1,0.1){};
 \node [neuron missing2] at (-1,-1.9){};
 
 \node [neuron missing]  at (0,-1) {};
 \node [neuron missing3]  at (-1,-1) {};
 \node [neuron missing]  at (-2,-1) {};
 \node [neuron missing]  at (-3,-1) {};

\foreach \m/\l [count=\y] in {1,2}
{
 \node [circle,fill=black!50,minimum size=1cm] (hidden-\m) at (2,2-\y) {};
}
\foreach \m/\l [count=\y] in {3}
{
 \node [circle,fill=black!50,minimum size=1cm ] (hidden-\m) at (2,-2) {};
}
 
 \node [neuron missing4]  at (2,-1) {};

\foreach \m [count=\y] in {1}
  \node [circle,fill=black!80,minimum size=1cm ] (output-\m) at (4,-0.5) {};

\foreach \l [count=\i] in {1,2,3}
  \draw [<-] (hidden-\i) -- ++(-1.6,0)
    node [above, midway] {$\sigma\left(\inner{\bm w}{\cdot}\right)$};

\foreach \i in {1,2}
    \draw [->] (hidden-\i) -- (output-1)
     node [above, midway] {$v_\i$};
\draw [->] (hidden-3) -- (output-1) node [above,midway] {$v_m$};

\foreach \l [count=\i] in {1,2,m}
  \node [] at (input-\i) {$\bm Z_{\l}$};

\foreach \l [count=\i] in {1,2,m}
  \node [] at (hidden-\i) {$h_{\l}$};
  
\node [] at (4,0) {$\inner{\bm v}{\bm h}$};

\end{tikzpicture}
\caption{One-hidden-layer neural network. The first linear layer resembles a convolutional layer with each $\Z_i$ being a patch of size $n$ and $\w$ being the shared weights or filter. The second linear layer serves as the classifier.}
\label{fig:cnn}
\end{figure*}\label{fig1}

We fit the described model with quantized weights $\w\in\cQ$ and binary activation function $\sigma(x) = \mathds{1}_{\{x>0\}}$ on the i.i.d. Gaussian data $\{({\bm Z}, y^*_{\Z} ) \}_{{\bm Z}\sim \mathcal{N}(\mathbf{0},\mathbf{I})}$. In this paper, we will focus on the cases of binary and ternary weights. In the binary case, every quantized weight in $\w$ is either $\alpha$ or $-\alpha$ for some universal real-valued constant $\alpha>0$, or equivalently, $\cQ=\mathbb{R}_+\times\set{\pm1}^n$; this setup of binary weights is widely adopted in the literature; for example, \cite{rastegari2016xnor}. Similarly in the ternary case, we take $\cQ = \mathbb{R}_+\times\set{0,\pm1}^n$; see \cite{twn_16,Yin2016TrainingTN} for examples.

We use the squared loss to measure the discrepancy between the model output and label:
\begin{equation}\label{sample_loss}
\begin{aligned}
&\ell(\bm w; \bm \Z):=\frac{1}{2}\left(y(\bm Z; \bm w)- y^*_{\Z} \right)^2\\
=&\frac{1}{2}\left(\bm v^\top\sigma(\bm Z\bm w)-\bm v^\top\sigma(\bm Z\bm w^*)\right)^2.
\end{aligned}
\end{equation}
We cast the learning task as the following population loss minimization problem:
\begin{equation}\label{target}
\min_{\w\in\R^n}f(\bm w):=\mathbb{E}_{{\bm Z}\sim \mathcal{N}(\mathbf{0},\mathbf{I})}\left[\ell(\bm w;\bm Z)\right] \quad \mbox{subject to} \quad  \w\in\cQ
\end{equation}
where the sample loss  function $\ell(\bm w;\bm Z)$ is given in (\ref{sample_loss}).

In the rest of the paper, we study the convergence behavior of \ref{coarse} described below in Algorithm \ref{alg:bc} for solving optimization problem (\ref{target}), in which $\tilde{\nabla}f$ standards for an unusual gradient of $f$ called coarse gradient \cite{yin2018understanding}, so as to side-step the vanished gradient issue.
Since the loss function is scale-invariant, i.e., $\ell(\bm Z;\bm w)=\ell(\bm Z;\bm w/c)$ for any scalar $c>0$, without loss of generality,
we assume that $\norm{\bm w^*} =1$ is unit-normed.

\begin{algorithm}[htb]
\caption{\ref{coarse} algorithm for solving (\ref{target})}
\label{alg:bc}
\begin{algorithmic}
\STATE {\bfseries Input: number of iterations $T$, learning rate $\eta_t$, weight bits $b$.}
\STATE {\bfseries Initialize: auxiliary real-valued weights $\bm y^0\in\R^n$.}
\FOR{$t=1${\bfseries to}$T$}
\STATE $\bm y^t = \bm y^{t-1} - \eta_t\tilde{\nabla}f(\bm w^{t-1})$
\STATE  $\bm w^t = \proj{}{} (\bm y^t)$ 
\ENDFOR
\end{algorithmic}
\end{algorithm}
 Throughtout this paper we assume the following on the learning rate $\eta_t>0$:
\begin{enumerate}
    \item $\sum_{t=1}^\infty\eta_t=\infty$.
    \item $\eta_t$ is upper bounded by some positive constant $\eta$.
\end{enumerate}

\subsection{Characterization of Optimal Solutions}
To study the convergence of Algorithm \ref{alg:bc}, we first obtain the closed-form expression of the objective function for the optimization problem (\ref{sample_loss}), which only depends on the angle between quantized weight vector $\w$ and the true weight vector $\w^*$. This helps us find the expression of global minimum to (\ref{alg:bc}). 


\begin{lemma}\label{population_loss}
Let $\bm w\not=\bm 0$ be nonzero vector.
\begin{itemize}
    \item the training loss in (\ref{target}) is given by 
\begin{equation*}\label{simpleloss}
f(\bm w)=\frac{\norm{\bm v}^2}{2\pi}\arccos\left(\frac{\bm w^\top\bm w^*}{\norm{\bm w}}\right)
\end{equation*}
\item For any $\delta>0$, $\bm w=\delta \cdot\proj{}{} (\w^*)$ is a global optimum of quantization problem (\ref{target}).
\end{itemize}
\end{lemma}
The above result can be easily derived from Lemma 1 of \cite{yin2018understanding}, so we omit the proof. Lemma \ref{population_loss} states that the optimal quantized weights is just the projection of $\bm w^*$ onto $\cQ$, i.e., the direct quantization of teacher parameters $\w^*$. Note that the projection/quantization may not be unique, we refer to $\proj{}{}(\y)$ as any choice of the projection of $\y$ onto $\cQ$.

\subsection{Coarse Gradient}
In this part, we specify the coarse gradient $\tilde{\nabla}f(\bm w)$ in Algorithm \ref{alg:bc}. The standard back-propagation gives the gradient of $\ell(\w;\Z)$ w.r.t. $\bm w$ by 
\begin{equation*}\label{sample_grad}
    \nabla_{\bm w}\ell(\bm w; \bm Z)=\bm Z^\top\left(\sigma'(\bm Z\bm w)\odot\bm v\right)\ell(\bm w; \bm Z).
\end{equation*}
Note that $\sigma'$ is zero a.e., which makes $\nabla_{\bm w}\ell(\bm w; \bm Z)$ inapplicable to the training. The sample coarse gradient w.r.t. $\bm w$ associated with the sample $(\bm Z,y^*_{\bm Z})$ is given by replacing $\sigma'$ with a surrogate derivative, known as straight-through estimator (STE) \cite{bengio2013estimating,yin2018understanding}. Here we consider the derivative of ReLU function $\mu(x)=\max\{x, 0\}$ which is a widely used STE for quantization, namely, we modify the original gradient $ \nabla_{\bm w}\ell(\bm w; \bm Z)$ as follows:
\begin{equation*}\label{sample_cgrad}
    \tilde{\nabla}_{\bm w}\ell(\bm w; \bm \Z)=\bm Z^\top\left(\mu'(\bm Z\bm w)\odot\bm v\right)\ell(\bm w;\bm Z).
\end{equation*}
The coarse gradient induced by ReLU STE $\mu'$ is just the expectation of $\tilde{\nabla}_{\bm w}\ell(\bm w; \bm \Z)$ over $\bm Z \sim \mathcal{N}(\mathbf{0},\mathbf{I})$.
We evaluate the coarse gradient  $\tilde{\nabla}f(\bm w)$ used in Algorithm \ref{alg:bc}:
\begin{lemma}\label{cgrad}
The expected coarse gradient of $\ell(\bm w;\bm Z)$ w.r.t. $\bm w$ is
\begin{equation}\label{cgradeq}
\begin{aligned}
\tilde{\nabla}f(\bm w) :=&  \E_{\Z\sim \mathcal{N}(\mathbf{0},\mathbf{I})} [\tilde{\nabla}_{\bm w}\ell(\bm w; \bm \Z)] \\
=& \constv\left(\normv{\bm w}-\bm w^*\right).
\end{aligned}
\end{equation}
\end{lemma}

\subsection{Weight Quantization Step}
The following two lemmas give the closed-form formulas of the projection/quantization $\proj{}{}(\cdot)$ in Algorithm \ref{alg:bc} in the binary and ternary cases, respectively.

\begin{lemma}[Binary Case]\label{proj2}
For any non-zero $\bm y\in\mathbb{R}^n$, the projection of $\bm y$ onto $\cQ = \R_+ \times \{\pm 1\}^n$ is
$$
\proj{}{}(\y)=\frac{\norm{\bm y}_1}{n}\bsign{\bm y},
$$
where the sign function acts element-wise
$$
\bsign{\y}_i = 
\begin{cases}
1 & \mbox{if } \; y_i\geq 0 \\
-1 & \mbox{if } \; y_i< 0.
\end{cases}
$$
\end{lemma}

The above lemma is due to \cite{rastegari2016xnor}.
In the ternary case, \cite{Yin2016TrainingTN} gives the following result:

\begin{lemma}[Ternary Case]\label{proj3}
For any non-zero $\bm y\in\mathbb{R}^n$, the projection of $\bm y$ on $\cQ = \R_+ \times \{0, \pm 1\}^n$ is
$$\proj{}{} (\y)=\frac{\norm{\bm y_{\left[j^*\right]}}_1}{j^*}\sign{\bm y_{\left[j^*\right]}}$$
where 
$j^*=\argmax_{1\leq j\leq n}\frac{\norm{\bm y_{\left[j\right]}}_1^2}{j}$, and $\bm y_{[j]}\in\R^n$ extracts the first $j$ largest entries in magnitude of $\bm y$ and enforces $0$ elsewhere. Here,
$$
\sign{\y}_i = 
\begin{cases}
1 & \mbox{if } \; y_i> 0 \\
0 & \mbox{if } \; y_i= 0 \\
-1 & \mbox{if } \; y_i< 0.
\end{cases}
$$

\end{lemma}

\section{Main Results}
By Lemma \ref{population_loss}, we assume, for the ease of presentation, that the iterates $\{\w^t\}$ are normalized, that is,  we re-define $\w^t$ in Algorithm \ref{alg:bc} by
$$
\w^t=\nproj{}{}(\y^t):= \frac{\proj{}{}(\y^t)}{\norm{\proj{}{}(\y^t)}}
$$ 
Our results extend trivially to the original \ref{coarse} without normalization as the value of $f(\w)$ does not depend on $\|\w\|$. Furthermore, we denote by $\nproj{}{}(\w^*)$ the normalization of the quantization/projection of $\w^*$, $\proj{}{}(\w^*)$, which is a global minimum according to Lemma \ref{population_loss}.
Our main results show that the optimum $\nproj{}{}(\w^*)$ is recurrent as long as $\bm w^*$ is close to its normalized quantization. 

\begin{theorem}\label{recurrent}
Consider the setup of quantization problem (\ref{target}). Let $\cQ$ be either $\R_+ \times \{\pm 1\}^n$ (binary case) or $\R_+ \times \{0,\pm 1\}^n$ (ternary case).
There exists constant $\epsilon>0$ that depends on the weight bit-width and dimension $n$ only, such that for any $\w^*$ with $$0<\norm{\bm w^*-\nproj{}{}(\bm w^*)}<\epsilon,$$ we have $\bm w^t=\nproj{}{}(\w^*)$ for infinitely many $t$ values, where $\{\w^t\}$ is the sequence generated by Algorithm \ref{alg:bc} with any initialization. 
\end{theorem}
Intuitively, ternary weights should work better than binary weights. The following remark confirms this intuition by showing that the number of points where $\bm w^t$ visits infinitely many times is limited. 
\begin{rmk}\label{suboptimals}
In the ternary case, we can further prove that the sequence $\{\bm w^t\}$ generated by Algorithm \ref{alg:bc} has at most $n$ sub-sequential limits. 
\end{rmk}

\section{Proof Sketch}
On one hand, the binary case is rather simple. We show that part of the coordinates is stable while others have oscillating sign. We further prove that the set of oscillating coordinates is not empty as long as $\bm w^*\not\in\cQ = \R_+ \times \{\pm 1\}^n$ is not quantized.

On the other hand, the proof of the ternary case follows the following steps. Our first step shows the sequence $\bm y^t$ generated by Algorithm \ref{alg:bc} is bounded away from the origin for all but finitely many $t$ values. Then, our second step shows each coordinate of $\bm y^t$ is of the same sign of $\bm w^*$ for all but finitely many $t$ values. This forces $\bm y^t$ to stay in the same orthant to which $\bm w^*$ belongs. As a matter of fact, an $n$-dimensional space has in total $2^n$ orthants, which means $\bm y^t$ can only stay in a small region near $\bm w^*$. After that, our third step furthermore cuts the orthant into $n!$ congruent cones and argue $\bm y^t$ must stay in the same cone where $\bm w^*$ is for all but finitely many $t$ values. In the last step, we prove the ternary case of Theorem \ref{recurrent}, which asserts that as long as the underlying true parameter $\bm w^*$ is close to quantized state $\cQ = \R_+ \times \{0,\pm1\}^n$, i.e., any vertex of the cone it belongs to, the optimum is guaranteed to be recurrent. 

\subsection{Binary Weight}
In view of Lemmas \ref{population_loss} and \ref{proj2}, we have that the normalized optimum of (\ref{target}) is $\frac{1}{\sqrt{n}}\bsign{\bm w^*}$.
The Lemma below shows that some coordinates of $\bm w^t$  generated by Algorithm \ref{alg:bc} have oscillating signs.

\begin{prop}\label{nolimit1}
Let $\bm w^t$ be any infinite sequence generated by Algorithm \ref{alg:bc}.
If $|w_j^*|<\frac{1}{\sqrt{n}}$, then there exist infinitely many $t_1$ and $t_2$ such that $w_j^{t_1}=\frac{1}{\sqrt{n}}$ and $w_j^{t_2}=-\frac{1}{\sqrt{n}}$.
\end{prop}
The above lemma clearly implies that $\bm w^t$ does not converge, as long as $\bm w^*\not\in\cQ$. 
\begin{cor}\label{nolimit1cor}
If $\bm w^*\not\in\cQ$, then any sequence $\set{\bm w^t}$ generated by Algorithm \ref{alg:bc} does not converge. 
\end{cor}

Since Algorithm \ref{alg:bc} does not have a limit unless the weights in the network are already quantized, we ask a natural question: Can we guarantee the optimum to be visited infinitely  many times? The general answer is no. We have the following example {\it demonstrating that the optimum may never be achieved}. We refer the proof of the following example to the appendix. 
\begin{ex}\label{periodic}
Let $\bm w^*=\left(\frac{1}{6},\frac{1}{6},\frac{1}{6},\frac{1}{2}\sqrt{\frac{11}{3}}\right)$ so that the best the optimum $\nproj{}{}\bm w^*=\left(\frac{1}{2},\frac{1}{2},\frac{1}{2},\frac{1}{2}\right)$. Let $\eta_t=\eta$, $\lambda=\frac{\eta\norm{\bm v}^2}{6\sqrt{2\pi}}$ and
$$\left\{
\begin{aligned}
&y_1^0\in\left(-\lambda,0\right)\\
&y_2^0\in\left(0,\lambda\right)\\
&y_3^0\in\left(\lambda,2\lambda\right)\\
&y_4^0\in (0,\infty)
\end{aligned}\right.
$$
the sequence $\set{\bm w^t}$ generated by Algorithm \ref{alg:bc} with initialization $\bm y^0$ satisfies
$\bm w^{t+3}=\bm w^t$ and $\bm w^t\not=\nproj{}{}\bm w^*$ for all $t$.
\end{ex}

In the following, we give a sufficient condition for the optimum to be recurrent. The condition requires $\bm w^*$ to be close to $\cQ$. The following result is for the the binary case of Theorem \ref{recurrent}. 
\begin{repthm}{recurrent}[Binary Case]
If the optimum $$\hat{\bm w}:=\nproj{}{}(\w^*)=\frac{1}{\sqrt{n}}\bsign{\bm w^*}$$ of (\ref{target}) satisfies $$0<\sum_{|w_j^*|<\frac{1}{\sqrt{n}}}|w_j^*-\hat{w}_j|<\frac{2}{\sqrt{n}}$$
then there exist infinitely many $t$ values for any sequence $\set{\bm w^t}$ generated by Algorithm \ref{alg:bc} such that $\bm w^t=\nproj{}{}(\bm w^*)$.
\end{repthm}

\begin{figure*}[t]
    \centering
    \begin{tabular}{c}
         \includegraphics[width=0.98\linewidth]{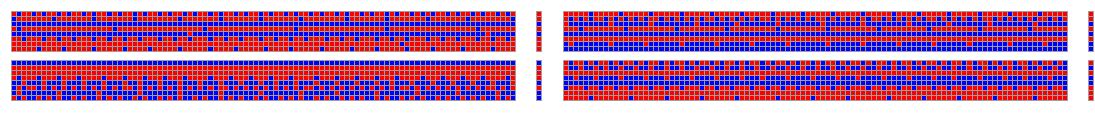} \\
         \includegraphics[width=0.98\linewidth]{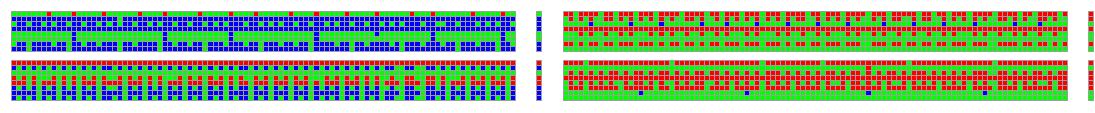}
    \end{tabular}
    \caption{{\bf Evolution of Weight signs of synthetic network described in (\ref{network}).} Each of the 8 large blocks is a colored display of weight sign values via $8\times 100$ matrix (i.e., 8 filter weight signs evolved over the last 100 iterations). The bars to the right of blocks are the corresponding optima.  {\bf Top two rows}: Binary weight signs,  red /blue for  $1$/$-1$.  {\bf Bottom two rows}: Ternary weight signs, red/green/blue for  $1$/$0$/$-1$. }
    \label{fig:toy}
\end{figure*}

\begin{figure*}[t]
    \centering
    \begin{tabular}{c}
         \includegraphics[width=0.9\linewidth]{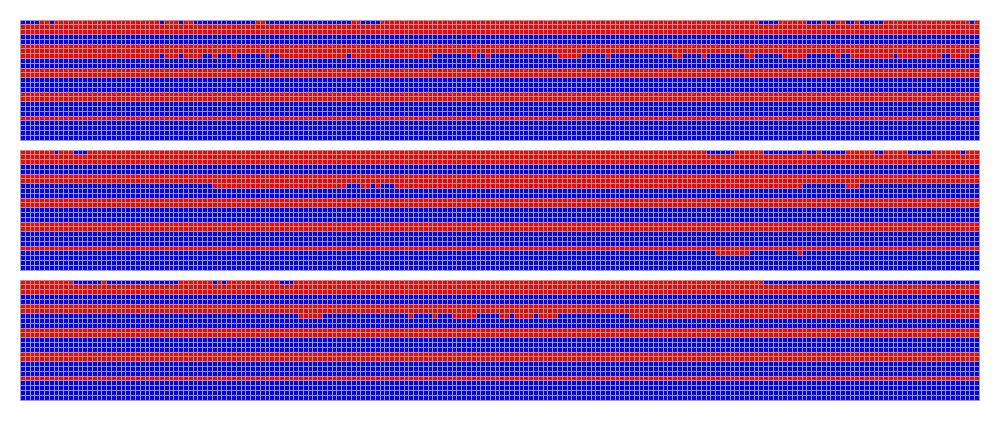} \\[-2ex]
         \includegraphics[width=0.9\linewidth]{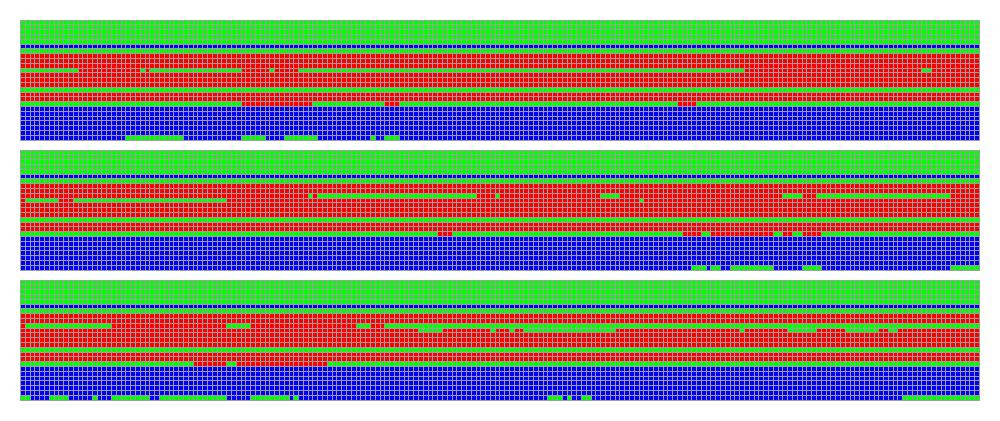}
    \end{tabular}
    \caption{{\bf  Evolution of signs of weight filters in the last training epoch (or 600 iterations) of LeNet-5.} 
    Each of the six $25\times200$ blocks corresponds to evolution of the $5\times5$ convolutional filter over $200$ iterations. {\bf Top three rows}: Binary weights over the last 600 iterations of training, red/blue for sign values $1$/$-1$.  {\bf Bottom three rows}: Ternary weights over the last 600 iterations of training, red/green/blue for sign values $1$/$0$/$-1$. }
    \label{fig:mnist}
\end{figure*}

\begin{figure*}[t]
    \centering
    \begin{tabular}{cc}
         \includegraphics[width=0.4\linewidth]{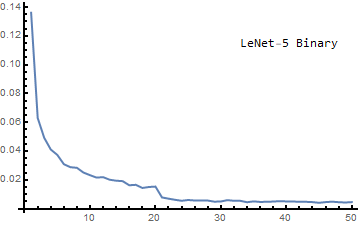} &
         \includegraphics[width=0.4\linewidth]{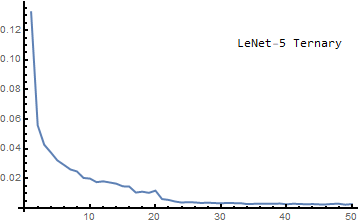}
    \end{tabular}
    \caption{LeNet-5 Training Loss v.s. Epoch. {\bf Left}: Binary weights.  {\bf Bottom}: Ternary weights. }
    \label{fig:lenet_loss}
\end{figure*}



\subsection{Ternary Weights}

The first result shows that $\bm w^t$ generated by Algorithm \ref{alg:bc} is generally divergent, and it converges only when the true parameters $\bm w^*\in\cQ = \R_{+}\times \{0, \pm 1\}^n$.

\begin{prop}[Ternary Case]\label{nolimit2}
Let $\set{\bm w^t}$ be any sequence generated by Algorithm \ref{alg:bc}. If $\bm w^*\not\in\cQ = \R_+ \times \{0,\pm 1\}^n$, then $\set{\bm w^t}$ is not a convergent sequence. 
\end{prop}

In what follows, we detail the proof of convergence behavior of Algorithm \ref{alg:bc}. 

Our first step is to rule out an exceptional  case that the direction of $\bm y^t$ changes significantly in only one iteration. As shown in Lemma \ref{cgrad}, the coarse gradient is bounded by a constant depending only on the fixed weight vector $\bm v$. So it suffices to show that  $\norm{\bm y^t}$ is bounded away from zero for all but finitely many $t$ values. 
\begin{lemma}\label{inf_norm}
Let $\set{\bm y^t}$ be any auxiliary sequence generated by Algorithm \ref{alg:bc}. If $\bm w^*\not\in\cQ$, then $\norm{\bm y^t}_1$ converges to infinity as $t$ increases.
\end{lemma}
Lemma \ref{inf_norm} shows that for any positive constant $c>0$, we have $\norm{\bm y^t}_1>c$ for all but finitely many $t$ values. 

Since Lemma \ref{inf_norm} guarantees that the direction of $\bm y^t$ will not change significantly, we cut down the region that $\bm y^t$ can belong to in two steps. To describe our first cut down, we need the following definition to make our statement precise. 
\begin{definition}\label{def:orthant}
For any $\bm x\in\mathbb{R}^n$, we define the orthant of $\bm x$ as
$$\bm O(\bm x):=\set{\bm y\in\mathbb{R}^n:\sign{\bm y}=\sign{\bm x}},$$
where $\sign{\cdot}$ acts coordinate-wise. Furthermore, we say $\bm O(\bm x)$ is regular if any coordinate of $\bm x$ is not zero. 
\end{definition}

We state some basic properties of the defined orthant.

\begin{prop}\label{prop:orthant}
For any $\bm x,\bm y\in\mathbb{R}^n$, the following statements are true: 
\begin{enumerate}
    \item Either $\bm O(\bm x)=\bm O(\bm y)$ or $\bm O(\bm x)\cap\bm O(\bm y)=\emptyset$.
    \item $\bm x\in\bm O(\bm x)$.
    \item $\cup_{\bm x\in\mathbb{R}^n}\bm O(\bm x)=\mathbb{R}^n$.
    \item There are in total $3^n$ orthants.
    \item There are in total $2^n$ regular orthants. 
\end{enumerate}
\end{prop}

\begin{lemma}\label{firstcut}
Let $\set{\bm y^t}$ be any auxiliary sequence generated by Algorithm \ref{alg:bc}.
If $\bm w^*\not\in\cQ^n$, then any subsequential limit of $\tilde{\bm y}^t:=\normv{\bm y^t}$ belongs to the closure of $\bm O(\bm w^*)$. Furthermore, if $\bm O(\bm w^*)$ is regular, then $\bm y^t$ lies in $\bm O(\bm w^*)$ for all but finitely many $t$ values. 
\end{lemma}

In our previous step, we have partitioned $\mathbb{R}^n$ into orthants and showed that $\bm y^t$ enter into a small neighborhood of the orthant where $\bm w^*$ stays. Now, we prove a stronger result based on the conclusion of our previous step. We would like to cut each orthant into several congruent cones which we shall define later and argue $\bm y^t$ will move and stay in close neighborhood of the cone where $\bm w^*$ stays. This step makes a stronger statement because we manage to shrink the size of the region where $\bm y^t$ can stay. 

\begin{definition}\label{def:cone}
For any non-zero vector $\bm x\in\mathbb{R}^n$, we define the cone of $\bm x$ to be
$$
\begin{aligned}
Cone(\bm x):=\bigg\{\bm y&\in\bm O(\bm x): \\
\sign{|y_j|-|y_i|}&=\sign{|x_j|-|x_i|} \mbox{ for } \forall i,j\in[n]\bigg\}.
\end{aligned}
$$
Moreover, we say $Cone(\bm x)$ is regular if $\bm O(\bm x)$ is regular and any $|x_j|\not=|x_i|$ for all $j\not=i$.
\end{definition}

\begin{prop}\label{prop:cone}
For any $\bm x,\bm y\in\mathbb{R}^n$, the following statements are true:
\begin{enumerate}
    \item Either $Cone(\bm x)=Cone(\bm y)$ or $Cone(\bm x)\cap Cone(\bm y)=\emptyset$.
    \item $\bm x\in Cone(\bm x)$.
    \item If $\bm y\in Cone(\bm x)$, then $Cone(\bm y)=Cone(\bm x)$.
    \item $\cup_{\bm y\in\bm O(\bm x)}Cone(\bm y)=\bm O(\bm x)$.
    \item Any regular orthant contains $n!$ regular cones. 
\end{enumerate}
\end{prop}

\begin{lemma}\label{secondcut}
Let $\set{\bm y^t}$ be any auxiliary real-valued sequence generated by Algorithm \ref{alg:bc}.
If $\bm w^*\not\in\cQ$, then any sub-sequential limit of $\tilde{\bm y}^t:=\normv{\bm y^t}$ belongs to the closure of $Cone(\bm w^*)$. Moreover, if $Cone(\bm w^*)$ is regular, then $\bm y^t\in Cone(\bm w^*)$ for all but finitely many $t$ values.
\end{lemma}

The auxiliary weight vector $\bm y^t$ can only stay in a small region around $\bm w^*$ for large $t$ values. 


\begin{definition}\label{def:vertex}
For any point $\bm x\in\mathbb{R}^n$, 
assume $(j_1,j_2,\cdots,j_n)$ is a permutation of $[n]$ such that
$$|x_{j_1}|\geq|x_{j_2}|\geq\cdots\geq|x_{j_n}|$$
We define the set of vertexes of $\bm x$ to be
$$
\begin{aligned}
&\Lambda(\bm x):=\\
&\set{\frac{1}{\sqrt{k}}\sum_{i=1}^k\sign{x_{j_i}}\bm e_{j_i}: x_{j_{k+1}}\not=x_{j_{k}} \mbox{ are nonzeros} }. 
\end{aligned}
$$
\end{definition}

Below are some basic facts about connection between vertexes and cones. 
\begin{prop}\label{prop:vertex}
For any $\bm x,\bm y\in\mathbb{R}^n$ let $k:=|\Lambda(\bm x)|$, the following statements are true: 
\begin{enumerate}
    \item $0\leq k\leq n$.
    \item $\Lambda(\bm x)$ is empty if and only if $\bm x=\bm 0$.
    \item $\Lambda(\bm x)$ is a subset of the boundary of $Cone(\bm x)$.
    \item $Cone(\bm x)=Cone(\bm y)$ if and only if $\Lambda(\bm x)=\Lambda(\bm y)$.
    \item $\nproj{}{} (\bm x)\in\Lambda(\bm x)$.
    \item $\bm y$ lies in $Cone(\bm x)$ if and only if there exists $k$ positive numbers $\set{\mu_{\bm z}(\bm y):\bm z\in\Lambda(\bm x)}$ such that 
    $$\bm y=\sum_{\bm z\in\Lambda(\bm x)}\mu_{\bm z}(\bm y)\bm z.$$
    \item $\bm y$ lies in the closure of $Cone(\bm x)$ if and only if there exists $k$ non-negative numbers $\set{\mu_{\bm z}(\bm y):\bm z\in\Lambda(\bm x)}$ such that 
    $$\bm y=\sum_{\bm z\in\Lambda(\bm x)}\mu_{\bm z}(\bm y)\bm z.$$
    \item $$\mathop{\cup}_{\bm x\in\mathbb{R}^n}\Lambda(\bm x)=\set{\bm x\in\cQ:\norm{\bm x}=1}.$$
\end{enumerate}
\end{prop}

\begin{lemma}\label{choices}
Let $\set{\bm w^t}$ be the sequence generated by Algorithm \ref{alg:bc}.
If $\bm w^*\not\in\cQ = \R_+ \times \{0,\pm 1\}^n$, then $\bm w^t\in\Lambda(\bm w^*)$ for all but finitely many $t$ values. 
\end{lemma}

The following result is the ternary case of Theorem \ref{recurrent} stated in section 3.

\begin{repthm}{recurrent}[Ternary Case]
Let $\set{\bm z_j}_{j=1}^k=\Lambda(\bm w^*)$ where ${\bm z_1}=\nproj{}{}\bm w^*$ is the optimum and $\bm w^*=\sum_{j=1}^k\lambda_j\bm z_j$. 
If $$0<\sum_{j=2}^k\lambda_j<1,$$ we have $\bm w^t=\nproj{}{}\bm w^*$ for infinitely many $t$ values, where $\bm w^t$ is any infinite sequence generated by Algorithm \ref{alg:bc} with any initialization. 
\end{repthm}

{\it Intuitively}, the parameter $\lambda_j$ in Theorem \ref{recurrent} stands for the proportion of time that $\{\w^t\}$ stays at $\bm z_j$. For instance, if $\lambda_j\approx1$, then most of $\{\w^t\}$ stay at $\bm z_j$ so that the oscillation has a longer `period' and is harder to observe. On the contrary, if all $\lambda_j$'s are almost the same then $\{\w^t\}$ behaves like uniform distribution and oscillation becomes more obvious. Beside $\lambda_j$'s, a smaller learning rate can render $\bm y^t$ moves slower which can also slow down the oscillation. Although there are ways to stabilize the training process, both our theorem and the experiments in the next section suggests the oscillation behavior is inevitable.  

\section{Experiments}
In this section, we implement \ref{coarse} algorithm on both synthetic data and MNIST/CIFAR image data. Our goals are (1) to validate our theoretical findings and (2) to show the appearance of the oscillation behavior in more complicated setups.
 \textbf{With that said, we emphasize that we did not extensively tune the hyper-parameters or use ad-hoc tricks to achieve the best possible validation accuracy.}  More comprehensive experimental results for \ref{coarse}-based approaches can be found in, for examples, \cite{halfwave_17,pact,Hubara2017QuantizedNN,dorefa_16}. Here we report the validation accuracies on MNIST and CIFAR-10 for fully quantized networks in Table \ref{tab:val_acc}.
For both synthetic and image data sets, we observed the oscillation behavior. 

\subsection{Synthetic Data}
We take $m=4$, $n=8$ in (\ref{network}) and construct $\bm v\sim N(\bm0,\bm I_m)$ and $\w^*\sim N(\bm0,\bm I_n)$ be random vectors. For each run, we fix $\bm v$ and $\bm w^*$ and train the neural network (\ref{network}) by algorithm (\ref{alg:bc}) for $200$ iterations with a learning rate being $0.1$. Fig. \ref{fig:toy} show the evolution of binary/ternary weight of $\w^t$ in the last $100$ iterations. Each block of size $8\times100$ corresponds to the evolution of $\w^t$ during the $100$ iterations. The (quantized) global minimum $\proj{}{}\w^*$ for each run is shown on the right side of the corresponding subplot in Fig. \ref{fig:toy}. 

\subsection{MNIST}

We train LeNet-5 with binary/ternary weights and 4-bit activations using \ref{coarse} algorithm. 
For deep networks, the (quantized) global optimum is generally unknown, we instead show  the oscillating behavior around local optimum. 
Note that Fig. \ref{fig:lenet_loss} shows the training loss no longer drop significantly during the last 30 epochs (50 in total). This suggests the network parameters have reached a local valley. 
However, Fig. \ref{fig:mnist} shows the iterating sequence of model parameters still have oscillating signs towards the end of training. 

Fig. \ref{fig:mnist} shows the evolution of the quantized weights of one convolution filter in the first convolution layer during the last 600 iterations.
To visualize the weights, each quantized filter is reshaped into a 25-dimensional column vector. Each block (3 in a group) of size $25\times200$ corresponds to the evolution of the one filter during 200 iterations.  As we can see from these two figures, a proportion of the weights do not converge to a limit but rather have oscillating signs. 

\subsection{CIFAR-10}

We repeat the  experiments on CIFAR-10 \cite{cifar_09} with ResNet-20/VGG-11. We train ResNet-20 \cite{resnet}/VGG-11 \cite{vgg_14} with binary/ternary weights and 4-bits activation using \ref{coarse} for $200$ epochs. We refer to the appendix for some figures that show similar oscillation behavior. Towards the end of training, although there has been no noticeable decay of training loss, we can still see the oscillating signs of the weights. 

\begin{table}[]
    \centering
    \begin{tabular}{|c|c|c|c|}\hline
                &   float   &   binary  &   ternary \\\hline
    LeNet-5     &  99.37    &   99.33   &   99.34   \\\hline
    ResNet-20   &  92.33    &   89.42   &   90.86   \\\hline
    VGG-11      &  92.15    &   89.47   &   90.91   \\\hline
    \end{tabular}
    \caption{Validation Accuracy of LeNet-5 on MNIST and ResNet-20/VGG-11 on CIFAR-10.}
    \label{tab:val_acc}
\end{table}

\section{Concluding Remarks}
We studied the convergence behavior of widely used \ref{coarse} algorithm \cite{bnn_16,halfwave_17,pact,dorefa_16} for the quantization of one-hidden-layer networks. We showed that the sequence of quantized weights $\{\w^t\}$ generated by \ref{coarse} is generically divergent if the teacher parameters are not in a quantized state, and constructed an explicit example of oscillatory divergence behavior. Under conditions that teacher parameters and their quantized values are close enough, we proved  the recurrence of \ref{coarse} algorithm at the global minimum. 

\section{Acknowledgement}
This work was partially supported by NSF grants IIS-1632935, DMS-1854434, DMS-1924548, and DMS-1924935.

\small

\bibliographystyle{icml2021}
\bibliography{reference}

\appendix
\clearpage
\section*{Appendix}
\begin{replemma}{cgrad}
The expected coarse gradient of $\ell(\bm w;\bm Z)$ w.r.t. $\bm w$ is
\begin{equation*}
\tilde{\nabla}f(\bm w)=\constv\left(\normv{\bm w}-\bm w^*\right).\tag{\ref{cgradeq}}
\end{equation*}
\end{replemma}
\begin{proof}[Proof or Lemma \ref{cgrad}]
\cite{yin2018understanding} gives
$$\tilde{\nabla}f(\bm w)=\frac{\norm{\bm v^*}^2}{\sqrt{2\pi}}\left(\normv{\bm w}-\cos\left(\frac{\theta}{2}\right)\normv{\normv{\bm w}+\bm w^*}\right).$$
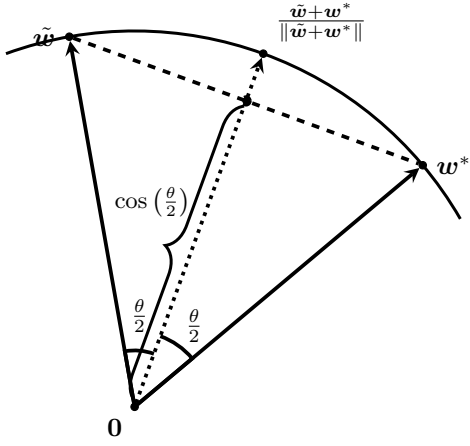
\begin{figure}[ht]
    \centering
    \begin{tikzpicture}
	[
	scale=2.5,
	>=stealth,
	point/.style = {draw, circle,  fill = black, inner sep = 1pt},
	dot/.style   = {draw, circle,  fill = black, inner sep = .2pt},
	]
	\def\rad{2}
	\draw[line width=0.4mm] (30:\rad) arc (30:110:\rad);
	\node (O) at (0,0) [point, label={below left:$\bm 0$}]{};
	\node (wt) at +(100:\rad) [point,label={left:$\tilde{\bm w}$}]{};
	\node (ws) at +(40:\rad) [point,label={right:$\bm w^*$}]{};
	\node (mid) at +(70:\rad) [point,label={above right:$\normv{\tilde{\bm w}+\bm w^*}$}]{};
	\draw[->][line width=0.5mm] (O) -- (wt);
	\draw[->][line width=0.5mm] (O) -- (ws);
	\draw[dotted,->][line width=0.5mm] (O) -- (mid);
	\draw[dashed][line width=0.5mm] (wt) -- (ws) node (cg)[point, midway]{};
	\draw[line width=0.5mm] (0,0) -- (100:.3cm) arc (100:70:.3cm);
	\node (a1) at (85:0.3) [label={above:$\frac{\theta}{2}$}]{};
	\draw[line width=0.5mm] (0,0) -- (40:.4cm) arc (40:70:.4cm);
	\node (a2) at (55:0.3) [label={above right:$\frac{\theta}{2}$}]{};
	\draw [decorate,decoration={brace,amplitude=10pt}][line width=0.4mm] (O) -- (cg) node [black,midway,xshift=-0.5cm,yshift=0.7cm] {\footnotesize $\cos\left(\frac{\theta}{2}\right)$};
	\end{tikzpicture}
    \caption{2-dim section of $\mathbb{R}^n$ spanned by $\tilde{\bm w}$ and $\bm w^*$}
    \label{helperpic}
\end{figure}
Let $\tilde{\bm w}=\normv{\bm w}$, we can easily see from Fig. \ref{helperpic} that the coarse gradient can be further simplified as (\ref{cgradeq})
\end{proof}

\begin{repprop}{nolimit1}
Let $\bm w^t$ be any infinite sequence generated by Algorithm \ref{alg:bc}.
If $|w_j^*|<\frac{1}{\sqrt{n}}$, then there exist infinitely many $t_1$ and $t_2$ values such that $w_j^{t_1}=\frac{1}{\sqrt{n}}$ and $w_j^{t_2}=-\frac{1}{\sqrt{n}}$.
\end{repprop}
\begin{proof}[Proof of Lemma \ref{nolimit1}]
For notational simplicity, since $\norm{w_j^*}<\frac{1}{\sqrt{n}}$, we have
$$\alpha:=\frac{1}{\sqrt{n}}-w_j^*>0\A\beta:=\frac{1}{\sqrt{n}}+w_j^*>0.$$
Using Lemma \ref{proj2} in Algorithm \ref{alg:bc}, we see that
$$
\begin{aligned}
&y_j^{t+1}=y_j^t+\eta_t\constv(w_j^*-w_j^t)\\
=&y_j^t+\eta_t\constv\left(w_j^*-\frac{1}{\sqrt{n}}\bsign{y_j^t}\right),
\end{aligned}
$$
and thus
$$
y_j^{t+1}=\left\{
\begin{aligned}
y_j^t-\eta_t\constv\alpha&~~\text{ if }y_j^t\geq0\\
y_j^t+\eta_t\constv\beta&~~\text{ if }y_j^t<0
\end{aligned}
\right.
$$
Since $y_j^t$ is bounded for each fixed $t\geq0$ and $j\in[n]$, our desired result follows from our assumptions on learning rate $\eta_t$.
\end{proof}

\begin{repcor}{nolimit1cor}
If $\bm w^*\not\in\Q{}{}$, any sequence $\set{\bm w^t}$ generated by Algorithm \ref{alg:bc} does not converge. 
\end{repcor}
\begin{proof}[Proof of Corollary \ref{nolimit1cor}]
Since $\bm w^*\not\in\tQ{1}{n}$, we know there must exist some $j\in[n]$ such that $|w_j^*|<\frac{1}{\sqrt{n}}$ and Proposition \ref{nolimit1} gives our desired result.
\end{proof}

\begin{repex}{periodic}
Let $\bm w^*=\left(\frac{1}{6},\frac{1}{6},\frac{1}{6},\frac{1}{2}\sqrt{\frac{11}{3}}\right)$ so that the best the optimum $\nproj{}{}\bm w^*=\left(\frac{1}{2},\frac{1}{2},\frac{1}{2},\frac{1}{2}\right)$. Let $\eta_t=\eta$, $\lambda=\frac{\eta\norm{\bm v}^2}{6\sqrt{2\pi}}$ and
$$\left\{
\begin{aligned}
&y_1^0\in\left(-\lambda,0\right)\\
&y_2^0\in\left(0,\lambda\right)\\
&y_3^0\in\left(\lambda,2\lambda\right)\\
&y_4^0\in (0,\infty)
\end{aligned}\right.
$$
the sequence $\set{\bm w^t}$ generated by Algorithm \ref{alg:bc} with initialization $\bm y^0$ satisfies
$\bm w^{t+3}=\bm w^t$ and $\bm w^t\not=\nproj{}{}\bm w^*$ for all $t$.
\end{repex}
\begin{proof}[Proof of Example \ref{periodic}]
In order to show the periodicity, it suffices to show $w_j^{t+3}=w_j^t$. Note that
$\tilde{\partial}_{w_4}f(\bm w)<0$
we have $y_4^t>0$ for all $t$ since $w_4^0>0$. It follows that $w_4^{t}=w_4^0=\frac{1}{2}$. Next, we would like to show the periodicity of $w_j^t$ for $j\in[3]$.
Note that
$$y_j^{t+1}=\left\{
\begin{aligned}
&y_j^t+\frac{\eta\norm{\bm v}^2}{2\sqrt{2\pi}}\left(w_j^*+\frac{1}{2}\right)&&\text{ if }y_j^t<0\\
&y_j^t+\frac{\eta\norm{\bm v}^2}{2\sqrt{2\pi}}\left(-w_j^*+\frac{1}{2}\right)&&\text{ if }y_j^t\geq0\\
\end{aligned}
\right.$$
we choose $\bm w_j^*=\frac{1}{6}$ so that with 
$$\lambda=\frac{\eta\norm{\bm v}^2}{6\sqrt{2\pi}}$$ we have
$$y_j^{t+1}=\left\{
\begin{aligned}
&y_j^t+2\lambda&&\text{ if }y_j^t<0\\
&y_j^t-\lambda&&\text{ if }y_j^t\geq0\\
\end{aligned}
\right.$$
Hence, we have
$$\bm w^t=\left\{
\begin{aligned}
&\left(-\frac{1}{2},\frac{1}{2},\frac{1}{2},\frac{1}{2}\right)&&\text{ if }t\equiv0(\text{mod }3)\\
&\left(\frac{1}{2},-\frac{1}{2},\frac{1}{2},\frac{1}{2}\right)&&\text{ if }t\equiv1(\text{mod }3)\\
&\left(\frac{1}{2},\frac{1}{2},-\frac{1}{2},\frac{1}{2}\right)&&\text{ if }t\equiv2(\text{mod }3)\\
\end{aligned}
\right.$$
\end{proof}

\begin{repthm}{recurrent}[Binary Case]
If the optimum $\hat{\bm w}:=\nproj{1}{n}\bm w^*$ of (\ref{target}) satisfies $$\sum_{|w_j^*|<\frac{1}{\sqrt{n}}}|w_j^*-\hat{w}_j|<\frac{2}{\sqrt{n}}$$
then there exists infinitely many $t$ values for any sequence $\set{\bm w^t}$ generated by Algorithm \ref{alg:bc} such that $\bm w^t=\nproj{}{}(\w^*)$.
\end{repthm}
\begin{proof}[Proof of Theorem \ref{recurrent} on $b=1$]
Without loss of generality, we can assume $w_j^*\geq0$ for all $j\in[n]$ so that $\hat{w}_j=\frac{1}{\sqrt{n}}$ for all $j$.

Firstly, if $w_j^*>\frac{1}{\sqrt{n}}$, we know
$$y_j^{t+1}=y_j^t+\eta_t\constv\left(w_j^*-w_j^t\right)\geq w_j^t+\eta_t\constv\left(w_j^*-\frac{1}{\sqrt{n}}\right),$$
so that
$$y_j^t\geq y_j^0+\constv\left(\sum_{s=0}^{t-1}\eta_s\right)\left(w_j^*-\frac{1}{\sqrt{n}}\right)$$
where the right hand side goes to infinity and thus $w_j^t=\hat{w}_j$ for all but finitely many $t$ values.

Secondly, if $w_j^*=\frac{1}{\sqrt{n}}$, we know when $w_j^t<0$: 
$$
y_j^{t+1}=y_j^t+\eta_t\constv\left(w_j^*-w_j^t\right)= y_j^t+\eta_t\constv\frac{2}{\sqrt{n}}
$$
holds so that there must exist some $t$ such that $y_j^t>0$. Once $y_j^t>0$ we have $w_j^*=w_j^t$ so that $y_j^{t+1}=y_j^{t}$ and hence $w_j^t=\hat{w}_j$ for all but finitely many $t$ values.

Third, if $w_j^*<\frac{1}{\sqrt{n}}$, we have  $y_j^t\cdot\tilde{\partial}_jf(\bm w^t)>0$ so that $y_j^t$ is increasing when $y_j^t<0$ and decreasing when $y_k^t>0$. This tells us $y_j^t$ is bounded uniformly in $t$. Furthermore,
$$
\begin{aligned}
y_j^t=y_j^0+\constv\Bigg[&\left(\sum_{s=0}^{t-1}\mathds{1}_{\left\{w_j^s>0\right\}}\eta_s\right)\left(w_j^*-\frac{1}{\sqrt{n}}\right)\\
&+\left(\sum_{s=0}^{t-1}\mathds{1}_{\left\{w_j^s<0\right\}}\eta_s\right)\left(w_j^*+\frac{1}{\sqrt{n}}\right)\Bigg].
\end{aligned}
$$
For notation simplicity, we let
$$\alpha_j=\frac{1}{\sqrt{n}}-w_j^*>0\A\beta_j=w_j^*+\frac{1}{\sqrt{n}}>0,$$
$$a_j^t=\frac{1}{t}\sum_{s=0}^{t-1}\ind{w_j^s>0}\eta_s\A b_j^t=\frac{1}{t}\sum_{s=0}^{t-1}\ind{w_j^s<0}\eta_s.$$
Now, we have
$$\frac{y_j^t-y_j^0}{t}=\constv\left(-\alpha_ja_j^t+\beta_jb_j^t\right).$$
Since $y_j^t$ is bounded for all $w_j^*<\frac{1}{\sqrt{n}}$, we let $t\rightarrow\infty$ so that left hand side vanishes and
$$\lim_{t\rightarrow\infty}\frac{b_j^t}{a_j^t+b_j^t}=\frac{\alpha_j}{\alpha_j+\beta_j}.$$
By assumption, we have
$$\lim_{t\rightarrow\infty}\sum_{j=1}^n\frac{b_j^t}{a_j^t+b_j^t}=\sum_{j=1}^n\frac{\alpha_j}{\alpha_j+\beta_j}<1.$$
Hence, we know
$$\lim_{t\rightarrow\infty}\sum_{s=0}^{t-1}\ind{\w^s=\hat{\w}^*}\eta_s\geq \lim_{t\rightarrow\infty} \left[\left(1-\sum_{j=1}^n\frac{b_j^t}{a_j^t+b_j^t}\right)\sum_{s=0}^{t-1}\eta_s\right]=\infty,$$
where we used the assumption $\sum_{t=0}^\infty\eta_t=\infty$. Now, the desired result follows.
\end{proof}

\begin{repprop}{nolimit2}[Ternary Case]
Let $\bm w^t$ be any sequence generated by Algorithm \ref{alg:bc}. If $\bm w^*\not\in\cQ$, then $\set{\bm w^t}$ is not a converging sequence. 
\end{repprop}
\begin{proof}[Proof of Proposition \ref{nolimit2}]
We prove by contradiction. Observe that $\cQ\cap\mathcal{S}^{n-1}$ is a finite set, we know $\bm w^t$ converges to $\bm w^\infty$ is equivalent to $\bm w^t=\bm w^\infty$ for all but finitely many $t$ values. Assume $\bm w^t=\bm w^\infty$ for all but finitely many $t$ values, we know there exists some $T\geq0$ such that $\bm w^t=\bm w^\infty$ for all $t\geq T$. Thus,
$$
\begin{aligned}
\bm y^{T+t}&=\bm y^T-\sum_{s=0}^{t-1}\eta_{T+s}\tilde{\nabla}f\left(\bm w^{T+s}\right)\\
&=\bm y^T-\left(\sum_{s=0}^{t-1}\eta_{T+s}\right)\tilde{\nabla}f\left(\bm w^\infty\right)\\
&=\bm y^t+\left(\sum_{s=0}^{t-1}\eta_{T+s}\right)\constv\left(\bm w^*-\bm w^\infty\right).
\end{aligned}
$$
Now, we have
$$\inner{\bm y^{T+t}}{\bm w^\infty}=\inner{\bm y^T}{\bm w^\infty}+\left(\sum_{s=0}^{t-1}\eta_{T+s}\right)\constv\inner{\bm w^*-\bm w^\infty}{\bm w^\infty}$$
where 
$$\inner{\bm w^*-\bm w^\infty}{\bm w^\infty}=\inner{\bm w^*}{\bm w^\infty}-1<0.$$
Note that $\sum_{s=0}^{\infty}\eta_{T+s}=\infty$, there exists some $T_1(T)$, such that for all $t>T_1(T)$
$$\inner{\bm y^t}{\bm w^\infty}<0.$$
This contradicts Lemma \ref{proj3} and our desired result follows.
\end{proof}

\begin{replemma}{inf_norm}
Let $\set{\bm y^t}$ be any auxiliary sequence generated by Algorithm \ref{alg:bc}. If $\bm w^*\not\in\cQ$, then $\norm{\bm y^t}_1$ converges to infinity as $t$ increases. 
\end{replemma}
\begin{proof}[Proof of Lemma \ref{inf_norm}]
$\cQ\cap\mathcal{S}^{n-1}$ is a compact set because it is finite. Also, since $\cQ$ is symmetric, $\bm w^*\not\in\cQ$ also implies $-\bm w^*\not\in\cQ$. It follows that
$$\alpha:=\inf_{\bm w\in\cQ\cap\mathcal{S}^{n-1}}\theta\left(\bm w^*,\bm w\right)\in(0,\pi).$$
Hence, for any $\bm w\in\cQ\cap\mathcal{S}^{n-1}$ we have
$$\inner{-\tilde{\nabla}f(\bm w)}{\bm w^*}=\constv\inner{\bm w^*-\bm w}{\bm w^*}\geq\constv\left(1-\cos\alpha\right).$$
Now, we know
$$
\begin{aligned}
&\inner{\bm y^T}{\bm w^*}=\inner{\bm y^0}{\bm w^*}+\sum_{t=0}^{T-1}\eta_t\inner{-\tilde{\nabla}f\left(\bm w^t\right)}{\bm w^*}\\
&\geq\inner{\bm y^0}{\bm w^*}+\left(\sum_{t=0}^{T-1}\eta_t\right)\constv\cdot \left(1-\cos\alpha\right).
\end{aligned}
$$
Let $T\rightarrow\infty$, we see that $\lim_{t\rightarrow\infty}\norm{\bm y^t}=\infty$ which is equivalent to $\lim_{t\rightarrow\infty}\norm{\bm y^t}_1=\infty$.
\end{proof}

\begin{lemma}\label{zero}
Let $\bm w=\proj{}{}(\bm y)$, then $|y_j|<\frac{1}{5n}\norm{\bm y}_1$ implies $w_j=0$.
\end{lemma}
\begin{proof}[Proof of Lemma \ref{zero}]
Without loss of generality, we assume $y_i\geq0$ for all $i\in[n]$ and $y_j<\frac{1}{5n}\norm{\bm y}_1$ for a fixed $j\in[n]$. Let $\delta=\frac{1}{5n}\norm{\bm y}_1$ and
$$j_\delta:=|\set{i\in[n]:|y_i|\geq\delta}|$$
we know $j_\delta\geq1$ by the principle of drawer. 
Now, with
$$j^*=\argmax\frac{\norm{\bm y_{[j]}}_1^2}{j}$$
for any $1\leq k\leq n-j_\delta$
$$
\begin{aligned}
&\frac{\norm{\bm y_{[j^*]}}_1^2}{j^*}-\frac{\norm{\bm y_{[j_\delta+k]}}_1^2}{j_\delta+k}\\
\geq&\frac{\norm{\bm y_{[j_\delta]}}_1^2}{j_\delta}-\frac{\norm{\bm y_{[j_\delta+k]}}_1^2}{j_\delta+k}\\
=&\frac{\left(j_\delta+k\right)\norm{\bm y_{[j_\delta]}}_1^2-j_\delta\norm{\bm y_{[j_\delta+k]}}_1^2}{j_\delta\left(j_\delta+k\right)},
\end{aligned}
$$
where the numerator is
$$
\begin{aligned}
&k\norm{\bm y_{[j_\delta]}}_1^2-j_\delta\left(\norm{\bm y_{[j_\delta+k]}}_1^2-\norm{\bm y_{[j_\delta]}}_1^2\right)\\
\geq&k\left[\norm{\bm y_{[j_\delta]}}_1^2-j_\delta\delta\left(\norm{\bm y_{[j_\delta+k]}}_1+\norm{\bm y_{[j_\delta]}}_1\right)\right].
\end{aligned}
$$

With $\tau=\frac{\norm{\bm y_{[j_\delta+k]}}_1}{n\delta}$, we have
$$
\begin{aligned}
&k\left[\norm{\bm y_{[j_\delta]}}_1^2-j_\delta\delta\left(\norm{\bm y_{[j_\delta+k]}}_1+\norm{\bm y_{[j_\delta]}}_1\right)\right]\\
\geq&k\left[\left(\norm{\bm y_{[j_\delta+k]}}_1-k\delta\right)^2-2n\delta\norm{\bm y_{[j_\delta+k]}}_1\right]\\
=&k\left(n\delta\right)^2\left(\tau^2-4\tau+1\right).
\end{aligned}
$$
Note that 
$$
\tau=\frac{\norm{\y_{[j_\delta+k]}}_1}{n\delta}\geq\frac{\norm{\bm y}_1-n\delta}{n\delta}\geq4,
$$
we conclude that 
$$\frac{\norm{\bm y_{[j^*]}}_1^2}{j^*}>\frac{\norm{\bm y_{[j_\delta+k]}}_1^2}{j_\delta+k}$$
and hence $j^*\leq j_\delta$. Now, Lemma \ref{proj3} gives $w_j=0$.
\end{proof}

\begin{lemma}\label{orthant}
Let $\set{\bm w^t}$ and $\set{\bm y^t}$ be the sequence and the auxiliary sequence generated by algorithm \ref{alg:bc}.
Assume $\bm w^*\not\in\cQ$, the following statements hold.
\begin{itemize}
    \item If $w_j^*=0$, then $y_j^t$ is bounded and $w_j^t=0$ for all but finitely many $t$ values.
    \item If $w_j^*\not=0$, then $\sign{y_j^t}=\sign{w_j^*}$ for all but finitely many $t$ values.
\end{itemize}
\end{lemma}
\begin{proof}[Proof of Lemma \ref{orthant}]
On the one hand, we consider the case $w_j^*=0$, so that
$$y_j^{t+1}=y_j^t+\eta_t\constv\left(w_j^*-w_j^t\right)=y_j^t-\eta_t\constv w_j^t.$$
Note that Lemma \ref{proj3} shows $y_j^t$ and $w_j^t$ are of the same sign if $w_j^t\not=0$, we know $y_j^t$ is bounded by $C_j:=\max\set{|y_j^0|,\eta\constv}$. Moreover Lemma \ref{inf_norm} shows $\norm{\bm y^t}_1>5nC_j$ for all but finitely many $t$ values. Finally, we see from Lemma \ref{zero} that $w_j^t=0$ for all but finitely many $t$ values.

On the other hand, consider the case $w_j^*\not=0$. Without loss of generality, we can assume $w_j^*>0$. Note that whenever $y_j^t\leq0$, we also have $w_j^t\leq0$ so that
$$y_j^{t+1}=y_j^t+\eta_t\constv\left(w_j^*-w_j^t\right)\geq y_j^t+\eta_t\constv w_j^*.$$
From the above inequality, we see that $y_j^{t}$ is increasing where the increment is bounded from below by $\eta_t\constv w_j^*>0$ where $\sum\eta_t=\infty$, so that there must exist some $T_j>0$ such that $y_j^{T_j}>0$. With Lemma \ref{inf_norm}, we can without loss of generality assume that $\norm{\bm y^t}_1\geq5n\eta\constv$ for all $t\geq T_j$. For ease of notation, we let $\delta=\eta\constv$ so that $\norm{\bm y^t}_1\geq5n\delta$ for all $t\geq T_j$. We shall next prove that $y_j^t\geq0$ for all $t\geq T_j$. We prove by induction, assume $y_j^t>0$ for some $t>T_j$ and show $y_j^{t+1}>0$. 
\begin{enumerate}
    \item If $y_j^t>\delta$,
    $$y_j^{t+1}=y_j^t+\eta_t\constv\left(w_j^*-w_j^t\right)\geq y_j^t-\delta>0.$$
    \item If $0<y_j^t\leq\delta$, since $\norm{\bm y^t}_1\geq5n\delta$, Lemma \ref{zero} shows $w_j^t=0$ so that
    $$y_j^{t+1}=y_j^t+\eta_t\constv\left(w_j^*-w_j^t\right)=y_j^t+\frac{\eta_t\delta}{\eta} w_j^*>y_j^t>0.$$
\end{enumerate}
Combining the above two cases, we get our desired result.
\end{proof}

\begin{replemma}{firstcut}
Let $\set{\bm y^t}$ be any auxiliary sequence generated by Algorithm \ref{alg:bc}.
If $\bm w^*\not\in\cQ$, then any sub-sequential limit of $\tilde{\bm y}^t:=\normv{\bm y^t}$ belongs to the closure of $\bm O(\bm w^*)$. Furthermore, if $\bm O(\bm w^*)$ is regular, then $\bm y^t$ lies in $\bm O(\bm w^*)$ for all but finitely many $t$ values. 
\end{replemma}

\begin{proof}[Proof of Lemma \ref{firstcut}]
By Lemma \ref{orthant}, we see that $\sign{y_j^t}=\sign{w_j^*}$ for all $\bm w_j^*\not=0$. We only need to prove $w_j^*=0$ implies $\lim_{t\rightarrow\infty}\tilde{y}_j^t=0$. Indeed, by Lemma \ref{orthant}, we know that $y_j^t$ is bounded by $C_j$ while Lemma \ref{inf_norm} tells us $\norm{\bm y^t}$ goes to infinity. Thus, $\lim_{t\rightarrow\infty}\tilde{y}_j^t=\frac{y_j^t}{\norm{\bm y}}=0$.
\end{proof}

\begin{lemma}\label{order}
Let $\set{\bm w^*}$ and $\set{\bm y^t}$ be any sequence and auxiliary sequence generated by Algorithm \ref{alg:bc}.
Assuming that $\bm w^*\not\in\cQ_2^n$, we have the following fact.
\begin{enumerate}
    \item If $|w_j^*|>|w_i^*|$, then $|y_j^t|>|y_i^t|$ for all but finitely many $t$ values.
    \item If $|w_j^*|=|w_i^*|$, then $\left||y_j^t|-|y_i^t|\right|$ is bounded and $|w_j^t|=|w_i^t|$ for all but finitely many $t$ values.
\end{enumerate}
\end{lemma}
\begin{proof}[Proof of Lemma \ref{order}]
Without loss of generality, we can assume $w_1^*\geq w_2^*\geq\cdots\geq w_n^*\geq0$. 

For the first statement, we only need to show that $w_j^*>w_{j+1}^*$ implies $y_j^t>y_{j+1}^t$ for all but finitely many $t$ values. 
Note that whenever $y_j^t<y_{j+1}^t$, then Lemma \ref{proj3} implies $w_j^t\leq w_{j+1}^t$, hence
$$
\begin{aligned}
&y_j^{t+1}-y_{j+1}^{t+1}\\
=&\left(y_j^t+\eta_t\constv\left(w_j^*-w_j^t\right)\right)-\left(y_{j+1}^t+\eta_t\constv\left(w_{j+1}^*-w_{j+1}^t\right)\right)\\
=&\left(y_j^{t}-y_{j+1}^t\right)+\eta_t\constv\left[\left(w_j^*-w_{j+1}^*\right)+\left(w_{j+1}^t-w_j^t\right)\right]\\
\geq&\left(y_j^{t}-y_{j+1}^t\right)+\eta_t\constv\left(w_j^*-w_{j+1}^*\right).
\end{aligned}
$$
Now that we know $y_j^{t}-y_{j+1}^t$ is increasing as long as it is negative and $\sum\eta_t=\infty$. Therefore, we conclude that there exist infinitely many $t$ values such that $y_j^{t}-y_{j+1}^t>0$. We can therefore assume $y_j^{T}-y_{j+1}^T>0$, where $T$ is the constant in Lemma \ref{inf_norm} such that $\norm{\bm y^t}_1\geq5n\sqrt{2\epsilon}$ for all $t\geq T$ where we set $\epsilon=\frac{\eta\norm{\bm v^*}^2}{\sqrt{2\pi n}}$. Next, we would like to show $y_j^t-y_{j+1}^t>0$  for all $t\geq T$ by induction. 

Next, assuming $y_j^{t}-y_{j+1}^t>0$, we want to show $y_j^{t+1}-y_{j+1}^{t+1}>0$. 

On the one hand, if $y_j^{t}-y_{j+1}^t\geq\epsilon$, we have
$$
\begin{aligned}
&y_j^{t+1}-y_{j+1}^{t+1}\\
=&\left(y_j^{t}-y_{j+1}^t\right)+\eta_t\constv\left[\left(w_j^*-w_{j+1}^*\right)+\left(w_{j+1}^t-w_j^t\right)\right]\\
>&\left(y_j^{t}-y_{j+1}^t\right)+\eta_t\constv\left(w_{j+1}^t-w_j^t\right)\\
\geq&\left(y_j^{t}-y_{j+1}^t\right)-\eta_t\constv\frac{1}{\sqrt{n}}\geq\left(y_j^{t}-y_{j+1}^t\right)-\epsilon\geq0.
\end{aligned}
$$

On the other hand, if $y_j^t-y_{j+1}^t<\epsilon$, we still have
$$y_j^{t+1}-y_{j+1}^{t+1}>\left(y_j^{t}-y_{j+1}^t\right)-\eta_t\constv\left(w_j^t-w_{j+1}^t\right),$$
so that it suffices to show $w_j^t=w_{j+1}^t$. From Lemma \ref{proj3}, we see that with
\begin{equation}\label{cond0}
j^*=\argmax_{j\in[n]}\frac{\norm{\bm y_{[j]}^t}_1^2}{j},
\end{equation}
we only need to show $j\not=j^*$. We prove by contradiction, assuming $j=j^*$ so that $w_j^t>0$ and $w_{j+1}^t=0$. Lemma \ref{zero} shows $y_j^t\geq\frac{1}{5n}\norm{\bm y^t}_1$.
Also, (\ref{cond0}) gives
\begin{equation}\label{cond1}
\frac{\norm{\bm y^t_{[j-1]}}_1^2}{j-1}\leq\frac{\norm{\bm y^t_{[j]}}_1^2}{j}=\frac{\left(\norm{\bm y^t_{[j-1]}}_1+y_j^t\right)^2}{j}.
\end{equation}
Simplifying the above inequality, we get
$$\left(\frac{\norm{\bm y^t_{[j-1]}}_1}{y_j^t}\right)^2-2(j-1)\left(\frac{\norm{\bm y^t_{[j-1]}}_1}{y_j^t}\right)-(j-1)\leq0.$$
Left hand side is a quadratic function of $\left(\frac{\norm{\bm y^t_{[j-1]}}_1}{y_j^t}\right)$, we know
\begin{equation}\label{eq1}
\frac{\norm{\bm y^t_{[j-1]}}_1}{y_j^t}\leq j-1+\sqrt{j(j-1)}\leq n-1+\sqrt{n(n-1)}<2n.
\end{equation}
We write equation (\ref{cond1}) in a different way and get
\begin{equation}\label{cond1_}
j\geq\frac{\left(\norm{\bm y^t_{[j-1]}}_1+y_j^t\right)^2}{y_j^t\left(2\norm{\bm y^t_{[j-1]}}_1+y_j^t\right)}.
\end{equation}
Now, we use $j=j^*$ again, to get
\begin{equation}\label{cond2}
\frac{\norm{\bm y^t_{[j]}}_1^2}{j}\leq\frac{\norm{\bm y^t_{[j+1]}}_1^2}{j+1}.
\end{equation}
Rewriting the above inequality, we get
\begin{equation}\label{cond2_}
j\leq\frac{\left(\norm{\bm y^t_{[j-1]}}_1+y_j^t\right)^2}{y_{j+1}^t\left(2\norm{\bm y^t_{[j-1]}}_1+2y_j^t+y_{j+1}^t\right)}.
\end{equation}
Combining (\ref{cond1_}) and (\ref{cond2_}), we get
$$
\left(y_j^t-y_{j+1}^t\right)^2-\left(2\norm{\bm y^t_{[j-1]}}_1+4y_j^t\right)\left(y_j^t-y_{j+1}^t\right)+2\left(y_j^t\right)^2\leq0.
$$
Solving the above inequality, we get
\begin{equation}\label{cond3}
\begin{aligned}
y_j^t-y_{j+1}^t\geq&\norm{\bm y^t_{[j-1]}}_1 + 2 y_j^t \\
&- \sqrt{\norm{\bm y^t_{[j-1]}}_1^2 + 4 \norm{\bm y^t_{[j-1]}}_1 y_j^t + 2 (y_j^t)^2}.
\end{aligned}
\end{equation}
Combining (\ref{eq1}) and (\ref{cond3}), we get
\begin{equation}\label{cond4}
y_j^t-y_{j+1}^t\geq (y_j^t)^2 \left(2n+2-\sqrt{4n^2+4n+2}\right)>\frac{(y_j^t)^2}{2}.
\end{equation}
Recalling that $y_j^t\geq\frac{1}{5n}\norm{\bm y^t}_1\geq\sqrt{2\epsilon}$, we have
$$y_j^t-y_{j+1}^t>\frac{1}{2}\left(\frac{\norm{\bm y^t}_1}{5n}\right)^2>\epsilon.$$
This contradiction shows $j\not=j^*$, and hence $w_j^t=w_{j+1}^t$ and it follows that $y_j^{t+1}>y_{j+1}^{t+1}$. Now, we have proved our first statement. 

For the second statement, since $w_j^*=w_i^*$, we have
$$
\begin{aligned}
&y_j^{t+1}-y_{i}^{t+1}\\
=&\left(y_j^t+\eta_t\constv\left(w_j^*-w_j^t\right)\right)-\left(y_{i}^t+\eta_t\constv\left(w_{i}^*-w_{i}^t\right)\right)\\
=&y_j^{t}-y_{i}^t-\eta_t\constv\left(w_{j}^t-w_i^t\right)=y_j^{t}-y_{i}^t-2\frac{\eta_t\epsilon}{\eta}\left(w_{j}^t-w_i^t\right).\\
\end{aligned}
$$
Hence, we know that $|y_j^t-y_i^t|$ is bounded by
$$C_{i,j}:=\max\set{|y_j^0-y_i^0|,\frac{\eta\norm{\bm v^*}^2}{\sqrt{2\pi}}}.$$

Without loss of generality, we can assume $j<i$ and $\min\set{y_j^t, y_i^t}\geq0$ by Lemma \ref{orthant}. Recalling  (\ref{cond4}), we have $w_j^t\not=w_i^t$ implying that 
$$|y_j^t-y_i^t|>\frac{\max\set{y_j^t,y_i^t}^2}{2}\geq\frac{1}{2}\left(\frac{\norm{\bm y^t}}{5n}\right)^2$$
where the right hand side goes to infinity. This contradicts the boundedness of $|y_j^t-y_i^t|$ if there are infinitely many $t$ values such that $w_j^t\not=w_i^t$.
\end{proof}

\begin{replemma}{secondcut}
Let $\set{\bm y^t}$ be any auxiliary sequence generated by Algorithm \ref{alg:bc}.
If $\bm w^*\not\in\cQ$, then any sub-sequential limit of $\tilde{\bm y}^t:=\normv{\bm y^t}$ belongs to the closure of $Cone(\bm w^*)$. Moreover, if $Cone(\bm w^*)$ is regular, then $\bm y^t\in Cone(\bm w^*)$ for all but finitely many $t$ values.
\end{replemma}

\begin{proof}[Proof of Lemma \ref{secondcut}]
Note that we already have Lemma \ref{firstcut}, we only need to show for any sub-sequential limit $\bm y$ of $\tilde{\bm y}^t$, we have $\sign{|y_j|-|y_i|}=\sign{|w_j^*|-|w_i^*|}$. The first statement of Lemma \ref{order} tells us that it is true for all $\sign{|w_j^*|-|w_i^*|}\not=0$. Thus, it suffices to show that $|w_j^*|=|w_i^*|$ implies $|y_j|=|y_i|$.

Note that the second statement of Lemma \ref{order} says that $\left||y_j|-|y_i|\right|$ is bounded by $C_{i,j}$, while Lemma \ref{inf_norm} gives $\lim_{t\rightarrow\infty}\norm{\bm y^t}=\infty$, we see that
$$|\tilde{y}_j|=\lim_{k\rightarrow\infty}\frac{|y_j^{t_k}|}{\norm{\bm y^{t_k}}}=\lim_{k\rightarrow\infty}\frac{|y_i^{t_k}|}{\norm{\bm y^{t_k}}}=|\tilde{y}_i|.$$
\end{proof}

\begin{replemma}{choices}
Let $\set{\bm w^t}$ be the sequence generated by Algorithm \ref{alg:bc}.
If $\bm w^*\not\in\cQ$, then $\bm w^t\in\Lambda(\bm w^*)$ for all but finitely many $t$ values. 
\end{replemma}
\begin{proof}[Proof of Lemma \ref{choices}]
First, by Proposition \ref{prop:cone}, $\bm y^t\in Cone(\bm w^*)$ implies $\nproj{}{} (\y^t)\in\Lambda(\bm w^*)$. 

Second, let $\tilde{\partial}Cone(\bm w^*)=\overline{Cone(\bm w^*)}-Cone(\bm w^*)$.
Now, a non-zero $\bm y^t\in\tilde{\partial}Cone(\bm w^*)$ implies $Cone(\bm y^t)\subset\tilde{\partial}Cone(\bm w^*)$ so that we also have $\nproj{}{}(\bm y^t)\in\Lambda(\bm y^t)\subset\Lambda(\bm w^*)$.

Third, by compactness of $\overline{Cone(\bm w^*)}\cap\mathcal{S}^{n-1}$, we know there exists some $\epsilon>0$ such that $\tilde{\bm y}^t:=\normv{\bm y^t}$ lies in $\epsilon$-neighborhood of $Cone(\bm w^*)\cap\mathcal{S}^{n-1}$ implying  $\nproj{}{} (\bm y^t)\in\Lambda(\bm w^*)$.

Finally, Lemma \ref{secondcut} suggests $\tilde{\bm y}^t$ lies in $\epsilon$-neighborhood of $Cone(\bm w^*)$ for all but finitely many $t$ values. We get our desired result. 
\end{proof}

\begin{repthm}{recurrent}[Ternary Case]
Let $\set{\bm z_j}_{j=1}^k=\Lambda(\bm w^*)$ where ${\bm z}_1=\nproj{}{}\bm w^*$ is the optimum and $\bm w^*=\sum_{j=1}^k\lambda_j\bm z_j$. 
If $0<\sum_{j=2}^k\lambda_j<1$, we have $\bm w^t=\nproj{}{}\bm w^*$ for infinite many $t$ values, where $\bm w^t$ is any infinite sequence generated by Algorithm \ref{alg:bc} with any initialization. 
\end{repthm}
\begin{proof}[Proof of Theorem \ref{recurrent} (Ternary Case)]
Note that Lemma \ref{secondcut} suggests $\tilde{\bm y}^t=\normv{\bm y^t}$ lies in $\epsilon$-neighborhood of $Cone(\bm w^*)$ for all but finitely many $t$ values. 
Let $\Lambda(\bm w^*)=\set{\bm z_1,\cdots,\bm z_k}$ and define $\mu_j^t$ be the constants such that
$$\bm y^t=\sum_{j=1}^k\mu_j^t\bm z_j$$
which is determined uniquely by $\bm y^t$. 

Let $\bm w^t=\bm z_{j_t}$, we know from Algorithm \ref{alg:bc} that
$$\bm y^{t+1}-\bm y^t=\eta_t\constv(\bm w^*-\bm z_{j_t}).$$
Thus
$$\sum_{j=2}^k\mu_j^{t+1}=\sum_{j=2}^k\mu_j^{t}+\eta_t\constv\left[\left(\sum_{j=2}^k\lambda_j\right)-1\right].$$
It follows that
$$\sum_{j=2}^k\mu_j^{t}=\text{Constant}+\left(\sum_{s=0}^{t-1}\eta_s\right)\constv\left[\left(\sum_{j=2}^k\lambda_j\right)-1\right]<0,$$
for large $t$'s. 
Now we see that when $t$ is large enough, $\tilde{\bm y}^t$ is bounded away from $Cone(\bm w^*)$ which contradicts Lemma \ref{secondcut} and our desired result follows.
\end{proof}

\begin{figure*}[ht]
    \centering
    \begin{tabular}{cc}
    \includegraphics[width=0.48\linewidth]{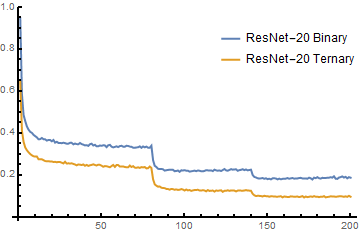} &
    \includegraphics[width=0.48\linewidth]{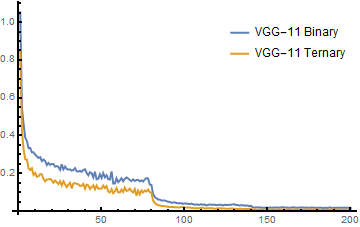} 
    \end{tabular}
    \caption{Training Loss of CIFAR-10. {\bf Left:} Binary/Ternary weight ResNet-20. {\bf Right:} Binary/Ternary weight VGG-11.}
    \label{fig:cifar_loss}
\end{figure*}

\begin{figure*}[t]
    \centering
    \includegraphics[width=0.9\linewidth]{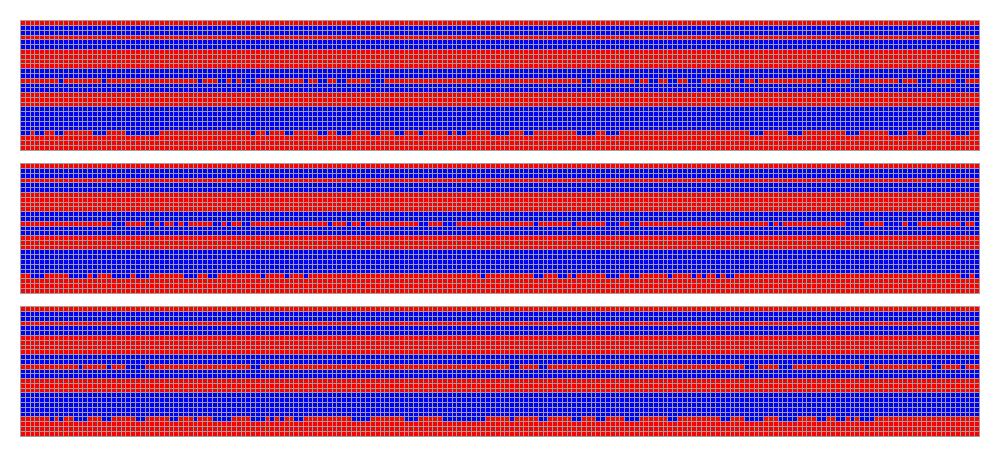}
    \caption{{\bf  Evolution of signs of weight filters in the last training epoch (or 600 iterations) of ResNet-20.} 
    Each of the three $27\times200$ blocks corresponds to evolution of the $3\times3\times3$ convolutional filter over $200$ iterations. Binary weights over the last 600 iterations of training, red/blue for sign values $1$/$-1$.}
\end{figure*}

\begin{figure*}[t]
    \centering
    \includegraphics[width=0.9\linewidth]{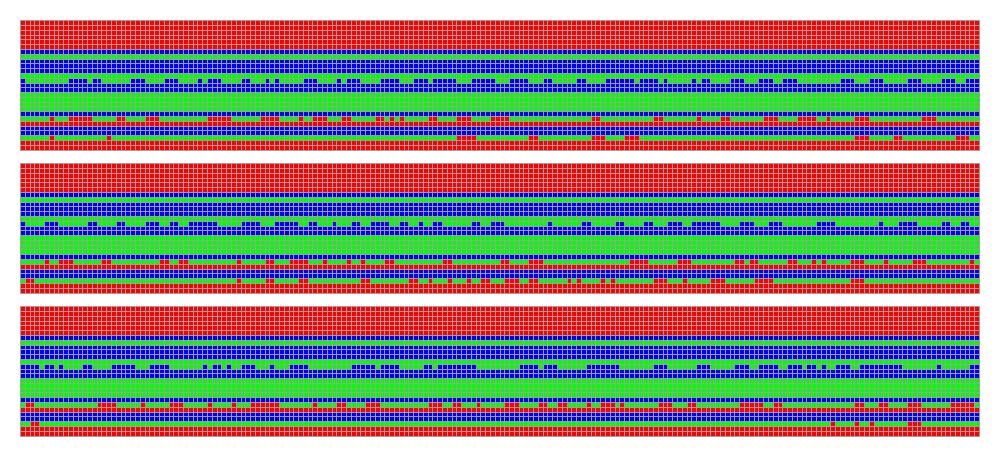}
    \caption{{\bf  Evolution of signs of weight filters in the last training epoch (or 600 iterations) of ResNet-20.} 
    Each of the three $27\times200$ blocks corresponds to evolution of the $3\times3\times3$ convolutional filter over $200$ iterations. Ternary weights over the last 600 iterations of training, red/green/blue for sign values $1$/$0$/$-1$.}
\end{figure*}

\begin{figure*}[t]
    \centering
    \includegraphics[width=0.9\linewidth]{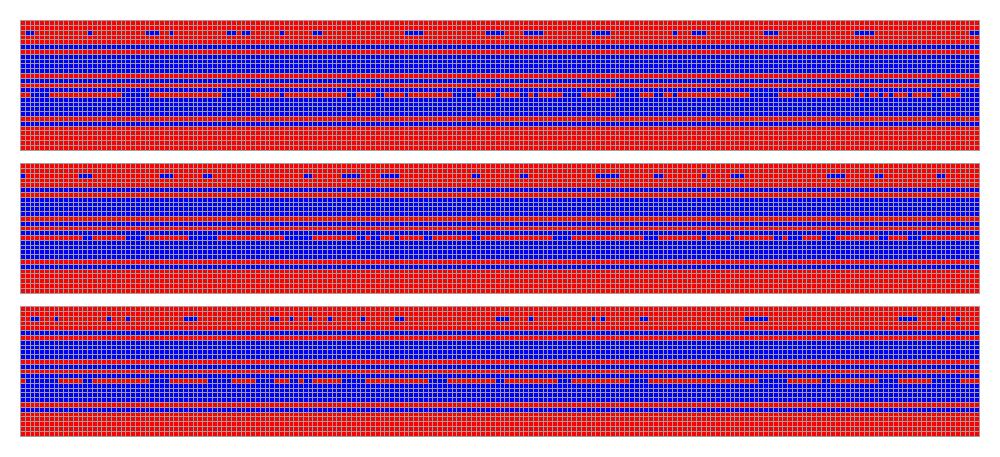}
    \caption{{\bf  Evolution of signs of weight filters in the last training epoch (or 600 iterations) of VGG-11.} 
    Each of the three $27\times200$ blocks corresponds to evolution of the $3\times3\times3$ convolutional filter over $200$ iterations. Binary weights over the last 600 iterations of training, red/blue for sign values $1$/$-1$.}
\end{figure*}

\begin{figure*}[t]
    \centering
    \includegraphics[width=0.9\linewidth]{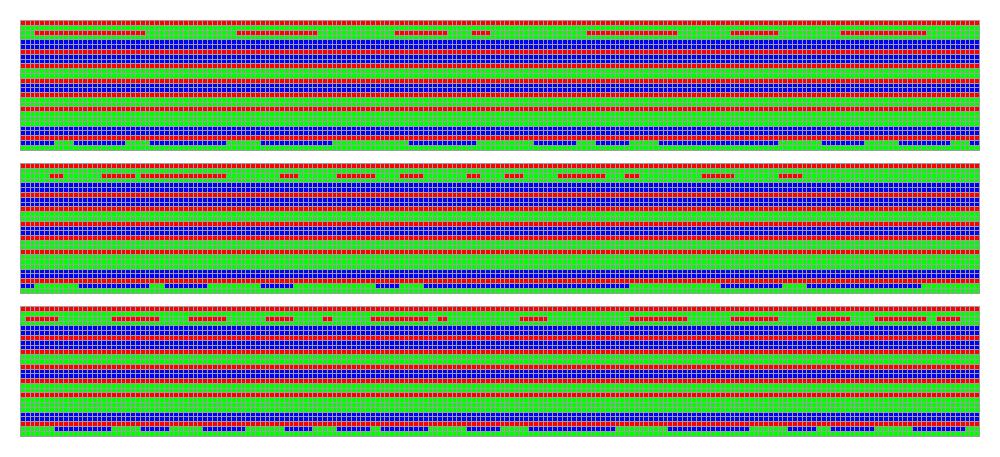}
    \caption{{\bf  Evolution of signs of weight filters in the last training epoch (or 600 iterations) of VGG-11.} 
    Each of the three $27\times200$ blocks corresponds to evolution of the $3\times3\times3$ convolutional filter over $200$ iterations. Ternary weights over the last 600 iterations of training, red/green/blue for sign values $1$/$0$/$-1$.}
\end{figure*}

\end{document}